%% file: cycles.tex
\documentclass{article} %
\usepackage{iclr2020_conference,times}
\pdfoutput=1

\input{math_commands.tex}

\usepackage{hyperref}
\usepackage{url}
\usepackage{amsthm}
\usepackage{thmtools}
\usepackage{thm-restate}
\usepackage{cancel}

\title{Cycles in Causal Learning}

\author{Katie Everett \\
Google Research \\
\texttt{everettk@google.com}
\And
Ian Fischer \\
Google Research \\
\texttt{iansf@google.com}
}

\iclrfinalcopy %
\begin{document}

\maketitle

\begin{abstract}

In the causal learning setting, we wish to learn cause-and-effect relationships between variables such that we can correctly infer the effect of an intervention. While the difference between a cyclic structure and an acyclic structure may be just a single edge, cyclic causal structures have qualitatively different behavior under intervention: cycles cause \emph{feedback loops} when the downstream effect of an intervention propagates back to the source variable. We present three theoretical observations about probability distributions with self-referential factorizations, i.e. distributions that could be graphically represented with a cycle. First, we prove that self-referential distributions in two variables are, in fact, independent. Second, we prove that self-referential distributions in $N$ variables have zero mutual information. Lastly, we prove that self-referential distributions that factorize in a cycle, also factorize as though the cycle were reversed. These results suggest that cyclic causal \emph{dependence} may exist even where observational data suggest \emph{independence} among variables. Methods based on estimating mutual information, or heuristics based on independent causal mechanisms, are likely to fail to learn cyclic casual structures. We encourage future work in causal learning that carefully considers cycles.

\end{abstract}

\section{Introduction}

\subsection{Motivation for Causal Learning}

Current methods in machine learning struggle to generalize beyond a test set drawn from a distribution identical to the training distribution. We would like to learn models that generalize more robustly: we want models that respond sensibly to adversarial attacks, out-of-distribution samples, and distributional shift. We can think of these tests for robustness as tests for causal learning - we are testing whether the model can adapt to a distribution that is indeed different from the training distribution, but that has the same underlying causal relationships between its high-level variables.

Following \citet{pearl2009causality}'s interventional definition of causal graphs, we define the goal in causal learning: we wish to learn enough about the underlying causal relationships between variables to be able to correctly predict the effect of an intervention~\citep{bengio2019meta}. A model that can predict the outcome of an intervention has learned the stable mechanisms underlying cause and effect, rather than memorizing spurious correlation in the data. Whether we actually perform the intervention is immaterial - a model that can imagine an intervention is performing the counterfactual reasoning necessary for decision making.

Whereas in the statistical learning setting, we can often achieve our objectives using directed acyclic graphical models, in the causal learning setting the need for cyclic models is clear: in any causal setting where intervention on a variable would cause a downstream effect on itself, the model requires a cyclic structure.

\subsection{Cycles in Statistical Learning}
In the statistical learning setting, the majority of literature on graphical models is based on directed acyclic graphs (DAGs). This confers a significant advantage: in DAGs, the global and local Markov properties are equivalent, even without assuming positivity of the probability distribution~\citep{lauritzen1990independence}. In contrast, in directed cyclic graphs, the global and local directed Markov properties are \emph{not} equivalent~\citep{richardson1997characterization}.

However, early work in the field of graphical models did consider cyclic graphs. \cite{spirtes1994conditional} and \cite{koster1996markov} showed that d-separation is valid on cyclic graphs, provided that the equations are linear and all distributions are Gaussian. \cite{pearl2013identifying} showed that d-separation generalizes to all directed cyclic graphs, provided the variables are discrete. We will rely on this result in the proof of Theorem 2.

In certain applications, the nature of observed data is sufficiently cyclic that it is worthwhile to sacrifice the Markov property in order to use cyclical graphical models. For example, feedback loops are common in neuronal spiking \citep{douglas2007recurrent}, economic activity \citep{moe1985control}, and gene transcription regulation networks \citep{alon2007network}. In time-series data, cycles may represent quantization, in which the time scale that data was collected was too coarse to see the individual events. Models such as Dynamic Bayesian networks handle time-series data by ``unrolling'' cycles in time, but with a tradeoff of space and data efficiency~\citep{ghahramani1997learning}.

\subsection{Cycles in Causal Learning}
In causal learning, we assume that most data in the real world is generated from the composition of a few independent causal mechanisms~\citep{pearl2009causality}~\citep{scholkopf2012causal}. \citet{peters2017elements} suggest that for this kind of data, there will exist a good factorization of high-level variables that mimics the true causal structure. This motivates the heuristic used in \citet{bengio2019meta} to learn a causal graph: assuming independent causal mechanisms, the heuristic independently parameterizes the likelihood of a certain variable being a parent (direct cause) of another variable.

However, in designing loss functions for causal learning, we should note that the difference between learning an acyclic causal structure and a cyclic causal structure, might be the existence of a single edge (representing a single parent-child relationship between variables). While this might trigger a small penalty in a loss function that treats edges independently, failing to learn a cycle in a causal structure is a significant qualitative error: we would be failing to expect a feedback loop in response to intervention on a variable within a cycle. Under intervention, we would like to know whether to expect a positive feedback loop or a negative feedback loop, and what kind of equilibrium, if any, the causal structure will reach.

We present the following theoretical results that suggest that cyclic causal structures may lurk in places we may not think to consider a cyclic structure, given knowledge of the observational distribution. In fact, to find cyclic causal structures, we may need to look for causal \emph{dependence} in places where observational data suggests complete \emph{independence} between variables. In particular, the mutual information on variables in a cycle is zero, suggesting that methods based on estimating mutual information are likely to overlook cycles.
\vfill
\pagebreak
\section {Results}

\begin{restatable}{theorem}{two_var_indp}
\label{thm:two_var_indp}
A two-variable probability distribution $P$ that factorizes according to $p(x,y) = p(x|y)p(y|x)$, also factorizes according to $p(x,y) = p(x)p(y)$.
\end{restatable}

\begin{proof}
Recall Equations (1) and (2), standard definitions in information theory \citep{cover2012elements}. From there, we proceed with substitution.
\begin{align*}
H(X,Y) &= H(X|Y) + H(Y|X) + I(X;Y) \tag{1} \\
H(X,Y) &= - \sum_{x,y} p(x,y)\log p(x,y) \tag{2}\\
 &= - \sum_{x,y} p(x|y)p(y|x)\log p(x|y)p(y|x) \\
 &= - \sum_{x,y} p(x|y)p(y|x)\log p(x|y) - \sum_{x,y} p(x|y)p(y|x)\log p(y|x) \\
 &= H(X|Y) + H(Y|X) \\
 \\
\text{Therefore,} \\
I(X;Y) &= 0\text{, which implies }p(x,y)=p(x)p(y).
\end{align*}
\end{proof}

\begin{restatable}{theorem}{nvarindp}
\label{thm:nvarindp}
A probability distribution P that factorizes as an n-variable cycle, according to $p(x_1, x_2, \ldots , x_n) = p(x_1|x_2)p(x_2|x_3) \ldots p(x_{n-1}|x_n)p(x_n|x_1)$, has mutual information $I(x_1; x_2; \ldots ; x_n) = 0$.
\end{restatable}
The full proof appears in \textsc{Appendix A}. Here we present a sketch of the proof.

\begin{proof}[Proof Sketch]
Our argument will follow a similar structure as the proof of Theorem 1. We first write the joint entropy as a sum of conditional entropy terms, by rewriting a log of products as a sum of logs (Lemma 1). Second, we write the joint entropy using its inclusion-exclusion definition from information theory (Lemma 2). We then show that the first representation is equal to all terms except the final joint mutual information term in the second representation.  We can therefore conclude that the joint mutual information is zero.

In the two-variable case, it was straightforward to show this equality because the inclusion-exclusion formula didn't contain any higher-order terms except for the joint mutual information. In the n-variable case, the inclusion-exclusion formula generates many terms.

To handle the additional terms, we recall that Shannon information obeys set additivity, which means that every information-theoretic identity has an equivalent set-theoretic identity. Following \cite{yeung1991new}, we define the information measure $\mu^*$ to be the unique signed measure that is consistent with Shannon's definitions of entropy and mutual information. This allows us to represent mutual information as the signed measure on the intersection of the sets we are inside of, conditioned on the union of the sets we are outside of. For ease of manipulating expressions, we write the proof using set-theoretic notation.

To finish the n-variable proof, we consider the set expressions representing the \emph{set difference} of each variable and its parent. We apply the inclusion-exclusion formula to the union of those set differences, and show that the higher-order terms fall into one of two cases. Either the term is the intersection of two adjacent set expressions (Lemma 3), which are disjoint, and therefore zero, or the term represents non-adjacent set expressions (Lemma 4), which are d-separated, and therefore zero. The only remaining terms are those corresponding to the conditional entropy terms in Lemma 2, leaving only the joint mutual information term which must now be zero.

\end{proof}
\begin{restatable}{theorem}{cycles_reverse}
\label{thm:cycles_reverse}
A joint distribution that factorizes in an n-variable cycle, also factorizes such that the n-variable cycle is reversed. In other words, $p(x_1, x_2, ..., x_n) = p(x_1|x_n)p(x_2|x_1)p(x_3|x_2)...p(x_n|x_{n-1})$  implies that $p(x_1, x_2, ..., x_n) = p(x_n|x_1)p(x_1|x_2)p(x_2|x_3)...p(x_{n-1}|x_n)$.
\end{restatable}
\begin{proof}
\begin{align*}
\text{Suppose} \\
p(x_1, x_2, ..., x_n) &= p(x_1|x_n)p(x_2|x_1)p(x_3|x_2) ... \\
&= p(x_1|x_n)\prod_{i=1}^{n-1}p(x_{i+1}|x_i) \\
\text{From Bayes' rule, we know that }\\
p(x_1|x_n) &= \frac{p(x_n|x_1)p(x_1)}{p(x_n)} \\
p(x_{i+1}|x_i) &= \frac{p(x_i|x_{i+1})p(x_{i+1})}{p(x_i)} \\
\text{Substituting gives us }\\
p(x_1, x_2, ..., x_n) &=  \frac{p(x_n|x_1)p(x_1)}{p(x_n)} \prod_{i=1}^{n-1} \frac{p(x_i|x_{i+1})p(x_{i+1})}{p(x_i)} \\
&= \Big[p(x_n|x_1)\prod_{i=1}^{n-1}p(x_i|x_{i+1})\Big] \cancelto{1}{\frac{p(x_1)}{p(x_n)} \prod_{i=1}^{n-1} \frac{p(x_{i+1})}{p(x_i)}}\\
&= p(x_n|x_1)\prod_{i=1}^{n-1}p(x_{i}|x_{i+1})
\end{align*}
\end{proof}

\section{Future Work}
The results presented herein apply to graphs where every variable is part of a single cycle: in the general case we are interested in all cyclic directed graphs, which may contain multiple cycles, overlapping cycles, and variables outside of the cycles. We leave these questions for future work.

\section{Conclusion}
In the statistical learning setting, we may be motivated to use cyclic graphs when the data overwhelmingly suggests that we should do so: when we are aware of feedback systems, or when it is natural to represent time-series data. In contrast, in the causal learning setting, we are motivated to learn the relationships between underlying variables so that we are able to infer the effect of an intervention. It is therefore imperative to learn a cyclic causal structure whenever one exists: that is, the case where intervening on a variable will have a downstream effect on itself.

However, the presence of cycles in a causal structure may not be obvious from the observation of their joint distribution. In fact, two variables that appear to be completely independent under observation, may indeed have a cyclic dependence present in their causal structure. A group of variables that shares no information between all its members, may contain a cyclic causal structure. And even if we learn that a cycle exists among a group of variables, we may not have distinguished the correct direction of the cycle.

Causal learning is an open area of research, and one that stands to benefit all of machine learning. Heuristics based on assuming independence between causal mechanisms may be a fruitful avenue for learning acyclic causal relationships, but may discourage learning cyclic structures where they exist. As we advance causal learning, we will need methods that treat cycles as first-class citizens.

\bibliography{iclr2020_conference}
\bibliographystyle{iclr2020_conference}
\vfill
\pagebreak
\appendix

\section{Proof of Theorem \ref{thm:nvarindp}}
For intuition, see the proof sketch appearing in the main paper.
We start by stating and proving the following lemmas.

\begin{restatable}{lemma}{sum_cond_ent}
\label{lem:sum_cond_ent}
$H(X_1, X_2, \ldots , X_n) = H(X_1|X_2) + H(X_2|X_3) + \ldots + H(X_{n-1}|X_n) + H(X_n|X_1)$
\end{restatable}
\begin{proof}
\begin{align*}
H&(X_1, X_2, \ldots, X_n) = - \sum_{x1, \ldots, x_n} p(x_1, \ldots, x_n) \log p(x_1, x_2, \ldots, x_n) \\
&= - \sum_{x1, \ldots, x_n} p(x_1, \ldots, x_n) \log p(x_1|x_2)p(x_2|x_3)\ldots p(x_n|x_1) \\
&= - \sum_{x1, \ldots, x_n} p(x_1, \ldots, x_n) \log p(x_1|x_2) - \ldots - \sum_{x1, \ldots, x_n} p(x_1, \ldots, x_n) \log p(x_n|x_1) \\
&= H(X_1|X_2) + H(X_2|X_3) + \ldots + H(X_n|X_1)
\end{align*}
\end{proof}

\begin{restatable}{lemma}{cond_plus_mi}
\label{lem:cond_plus_mi}
$H(X_1, X_2, ... , X_n) = \\ \mu^*\bigg((X_1-X_2) \cup (X_2 - X_3) \cup \ldots \cup (X_{n-1} - X_n) \cup (X_n - X_1) \bigg) + I(X_1; X_2; \ldots; X_n)$
\end{restatable}
\begin{proof}
We note that $X_1 \cap X_2 \cap \ldots \cap X_n$ is disjoint from $(X_1 - X_2) \cup (X_2 - X_3) \cup \ldots \cup (X_{n-1} - X_n) \cup (X_n - X_1)$.

That allows us to write \\
$X_1 \cup X_2 \cup \ldots \cup X_n = \big((X_1 - X_2) \cup (X_2 - X_3) \cup \ldots \cup (X_{n-1} - X_n) \cup (X_n - X_1)\big) \cup \big(X_1 \cap X_2 \cap \ldots \cap X_n\big)$.

We apply the information measure to both sides of this equation, because applying a measure to equal sets will give equal measures.
Therefore, we have
$H(X_1, X_2, \ldots , X_n) = \mu^*\big((X_1 - X_2) \cup (X_2 - X_3) \cup \ldots \cup (X_{n-1} - X_n) \cup (X_n - X_1)\big) + I(X_1; X_2; \ldots ; X_n)$.

\end{proof}

\begin{restatable}{lemma}{adj_disjoint}
\label{lem:adj_disjoint}
\text{Adjacent set expressions are disjoint:} If $k=j \oplus 1$ or $j=k \oplus 1$, then $\mu^*(\bigcap\limits_{i \in W}(X_i - X_{i \oplus 1})) = 0$
\end{restatable}
\begin{proof}
First, $j,k \in W$ implies
\begin{align*}
\bigcap_{i \in W}(X_i - X_{i \oplus 1}) &\subseteq (X_j - X_{j \oplus 1}) \cap (X_k - X_{k \oplus 1})\\
\text{And, for } k = j \oplus 1 \\
(X_j - X_{j \oplus 1}) \cap (X_k - X_{k \oplus 1}) &= (X_j - X_k) \cap (X_k - X_{k \oplus 1}) \\
&= X_j \cap X_k^c \cap X_k \cap X_{k \oplus 1}^c \\
&= \emptyset \\
\text{with a symmetric argument holding for} j = k \oplus 1. \\
\text{Therefore,} \\
\bigcap_{i \in W}(X_i - X_{i \oplus 1}) &\subseteq \emptyset \text{, so} \\
\bigcap_{i \in W}(X_i - X_{i \oplus 1}) &= \emptyset \text{, so} \\
\mu^*(\bigcap_{i \in W}(X_i - X_{i \oplus 1})) &= \mu^*(\emptyset) = 0.
\end{align*}
\end{proof}

\begin{restatable}{lemma}{nonadj_dsep}
\text{Non-adjacent set expressions are d-separated:} If $k \neq j \oplus 1$ and $j \neq k \oplus 1$, then $\mu^*(\bigcap\limits_{i \in W}(X_i - X_{i \oplus 1}) = 0$.
\label{lem:nonadj_dsep}
\begin{proof}
If $k \neq j \oplus 1$ and $j \neq k \oplus 1$, it implies that $j$, $k$, $j \oplus 1$, and $k \oplus 1$ are four distinct values, corresponding to four distinct variables in the cycle.

We are interested in whether $X_j$ and $X_k$ are d-separated by $X_{j \oplus 1}$, $X_{k \oplus 1}$.

There are exactly two paths between $X_j$ and $X_k$, neither path containing a collider. $X_{j \oplus 1}$ is on the opposite path from $X_{k \oplus 1}$. Therefore, conditioning on $X_{j \oplus 1}$ and $X_{k \oplus 1}$ will d-separate $X_j$ and $X_k$.

More formally, this d-separation condition applies to
$\mu^*\big(\bigcap\limits_{a \in A} a - \bigcap\limits_{b \in B} b \big) = 0$ for any $A$, $B$ such that $\{X_j, X_k\} \subseteq A$ and $\{X_{j \oplus 1}, X_{k \oplus 1} \} \subset B$.

We can rewrite
$\mu^*\big(\bigcap\limits_{i \in W}(X_i - X_{i \oplus 1}\big) = \mu^*\big(\bigcap\limits_{i \in W} X_i - \bigcap\limits_{i \in W} X_{i \oplus 1} \big)$, where $\{X_j, X_k\} \subseteq \bigcap\limits_{i \in W} X_i$ and $\{X_{j \oplus 1}, X_{k \oplus 1}\} \subseteq \bigcap\limits_{i \in W} X_{i \oplus 1}$.

That allows us to apply the d-separation condition, giving
$\mu^*\big(\bigcap\limits_{i \in W}(X_i - X_{i \oplus 1} \big) = 0$.

\end{proof}
\end{restatable}

Now we can give the proof of Theorem~\ref{thm:nvarindp}, which we restate for convenience:

\nvarindp*

\begin{proof}
We consider a probability distribution P that factorizes according to $p(x_1, x_2, \ldots , x_n) = p(x_1|x_2)p(x_2|x_3) \ldots p(x_{n-1}|x_n)p(x_n|x_1)$. We will write the joint entropy in two forms, in Lemma 1 and Lemma 2.

It remains to show that $\mu^*\big((X_1-X_2) \cup (X_2 - X_3) \cup \ldots \cup (X_{n-1} - X_n) \cup (X_n - X_1) \big) = H(X_1|X_2) + H(X_2|X_3) + \ldots + H(X_{n-1}|X_n) + H(X_n|X_1)$, which will allow us to conclude that $I(X_1; X_2; \ldots; X_n)= 0$.

To do this, we will apply the inclusion-exclusion formula to $\mu^*\big((X_1-X_2) \cup (X_2 - X_3) \cup \ldots \cup (X_{n-1} - X_n) \cup (X_n - X_1) \big)$, which we can do because it is a measure on a union of sets.

Informally, the inclusion-exclusion formula will produce many terms. The terms containing one set difference, correspond to the conditional entropy terms. The higher-order terms containing an intersection of at least two set differences, will fall into one of two cases, but in both cases the term will equal zero. The first case we handle in Lemma 3 and the second case we handle in Lemma 4.

To be more formal, let $\mathbb{W}$ be the power set on $\{i \in \mathbb{N}, i < n\}$. The inclusion-exclusion formula gives us a sum of terms, where each term is  $\mu^*(\bigcap\limits_{i \in W}(X_i - X_{i \oplus 1}))$ for some $W \in \mathbb{W}$. We can ignore the negative signs on higher-order terms because we will soon show these terms are all zero.

The higher-order terms are those where $|W| \geq 2$. For these terms, we consider $\mu^*(\bigcap\limits_{i \in W}(X_i - X_{i \oplus 1}))$ where $\oplus$ indicates addition mod n. Since $|W| \geq 2$, without loss generality we pick any two distinct elements of $W$ and call them $j$,$k$. If
$k=j \oplus 1$ or $j=k \oplus 1$, we apply Lemma 3 to show the term is zero. If $k \neq j \oplus 1$ and $j \neq k \oplus 1$, we apply Lemma 4 to show the term is zero.

We have now shown that all the terms with $|W| \geq 2$ are zero, which leaves only the terms with $|W| = 1$, giving $\mu^*\big((X_1-X_2) \cup (X_2 - X_3) \cup \ldots \cup (X_{n-1} - X_n) \cup (X_n - X_1) \big) = H(X_1|X_2) + H(X_2|X_3) + \ldots + H(X_{n-1}|X_n) + H(X_n|X_1)$. Applying Lemma 1 and Lemma 2 leaves us with $I(X_1; X_2; \ldots; X_n)= 0$, as desired.
\end{proof}

\end{document}

%% file: math_commands.tex
\usepackage{amsmath,amsfonts,bm}

\def\eqref#1{equation~\ref{#1}}

\def\1{\bm{1}}

\DeclareMathAlphabet{\mathsfit}{\encodingdefault}{\sfdefault}{m}{sl}
\SetMathAlphabet{\mathsfit}{bold}{\encodingdefault}{\sfdefault}{bx}{n}